%% file: main.tex
\documentclass[11pt]{article}

\usepackage[final]{acl}

\usepackage{times}
\usepackage{latexsym}
\usepackage[T1]{fontenc}
\usepackage[utf8]{inputenc}
\usepackage{microtype}
\usepackage{inconsolata}

\usepackage{graphicx}
\usepackage{hyperref}
\usepackage{url}
\usepackage{booktabs}
\usepackage{amsfonts}
\usepackage{nicefrac}
\usepackage{xcolor}
\usepackage{amsmath}
\usepackage{amsthm}
\usepackage{algorithm}
\usepackage{algpseudocode}
\usepackage[section]{placeins}
\algdef{SE}[INDENT]{Indent}{EndIndent}{}{}
\usepackage{subcaption}
\usepackage{enumitem}
\usepackage{multirow}

\newtheorem{theorem}{Theorem}
\newtheorem{proposition}[theorem]{Proposition}
\newtheorem{lemma}{Lemma}

\theoremstyle{definition}

\theoremstyle{remark}

\DeclareMathOperator*{\argmin}{arg\,min}
\DeclareMathOperator*{\argmax}{arg\,max}
\newcommand{\indicator}{\mathbb{1}}
\newcommand{\expect}{\mathbb{E}}
\newcommand{\Acc}{\mathrm{Acc}}
\newcommand{\Accstar}{\Acc^*}

\title{ESPO: Error-Structured Prompt Optimization\\via Diagnose, Diversify, and Stabilize}

\author{Lihao Liu, Peng Tang, Kunwar Yashraj Singh, Shabnam Ghadar \\
        AWS Agentic AI}

\begin{document}

\maketitle

\begin{abstract}
Evolutionary prompt optimizers such as GEPA suffer from \emph{prompt bloat}: each iteration appends rules and caveats, producing prompts up to $3\times$ longer yet no more accurate. We trace this to three deficiencies - incomplete error observation, limited search diversity, and unreliable selection - and propose \textbf{ESPO} (Error-Structured Prompt Optimization), which decomposes prompt optimization into three phases: \textbf{Diagnose} clusters all training errors into structural patterns in one round; \textbf{Propose} generates candidates via four complementary strategies with independent biases; \textbf{Select} applies bootstrap stability selection. On seven public NLP benchmarks - Tweet, MMLU, GSM8K, HotpotQA, ScoNe, HoVer, and PUPA - ESPO improves average accuracy by $+$3.76 pp over the state-of-the-art (74.67\% vs.\ 70.91\% for GEPA), matching or exceeding GEPA on every dataset while producing prompts 47\% shorter (1{,}004 vs.\ 1{,}878 chars) and faster at inference. Cross-model experiments across four additional student models (Gemma~3 12B, Mistral 14B, Qwen3 32B, Claude Haiku 4.5) show ESPO yields the best average accuracy on every model tested, with the largest gap on Qwen3 GSM8K (15.00\% $\to$ 91.40\%). A generalization bound (Appendix) grounds each phase in a corresponding term of the test-time gap, and the ablation confirms a key prediction: adding diversity \emph{without} bootstrap selection actually hurts performance ($-$1.20\%).
\end{abstract}

\input{sections/introduction}
\input{sections/related_work}
\input{sections/method}

\input{sections/experiments}
\input{sections/conclusion}
\input{sections/limitations}

\bibliography{bib/references}

\clearpage
\appendix
\input{sections/appendix}

\end{document}

%% file: sections/introduction.tex
\section{Introduction}
\label{sec:introduction}

Large language models are increasingly steered through natural-language prompts rather than parameter updates~\citep{brown2020language, wei2022chain}. As tasks grow more complex, automated prompt optimizers such as APE~\citep{zhou2023large}, OPRO~\citep{yang2024large}, MIPROv2~\citep{opsahl2024optimizing}, and GEPA~\citep{agrawal2026gepa} have emerged to replace manual trial-and-error with systematic search. Among these, GEPA achieves state-of-the-art results through an evolutionary loop: it evaluates prompts on training examples, reflects on errors using a strong LLM, proposes mutations, and selects survivors via a Pareto front.

A persistent shortcoming of evolutionary prompt optimization is \emph{\textbf{prompt bloat}}: each iteration tends to append rules and caveats, so prompts grow without bound. In our experiments, GEPA produces prompts up to $3\times$ longer than ESPO's (1.87$\times$ on average) yet less accurate from the same weak start, increasing latency and token cost, and can \emph{harm} generalization by overfitting to training noise. We trace this to three structural deficiencies in the evolutionary paradigm. (i)~\textbf{Incomplete error observation:} GEPA reflects on 3--8 random errors per round; by a coupon-collector argument, observing all systematic failure patterns with 95\% probability needs ${\sim}15$ rounds, and in the meantime the optimizer accumulates redundant or contradictory rules addressing symptoms rather than root causes. (ii)~\textbf{Limited search diversity:} a single mutation operator locks in one bias profile; when it fails on certain error types, it compensates by stacking more rules - inflating length without improving accuracy. (iii)~\textbf{Unreliable selection:} point estimation on a small (${\sim}30$-example) validation set with ${\sim}10$ candidates is a multiple-testing problem, and the optimizer may pick a verbose candidate that ranks first by noise rather than true quality.

We propose \textbf{ESPO} (Error-Structured Prompt Optimization), a framework that recasts prompt optimization from evolutionary search to structured statistical estimation. ESPO decomposes optimization into three principled phases. The \emph{Diagnose} phase analyzes \emph{all} training errors and clusters them into 3--7 structural patterns, achieving complete error coverage in a single round. The \emph{Propose} phase generates candidates through $K{=}4$ complementary strategies - diagnostic revision, consolidation, ablation, and factual injection - each addressing different error types with different inductive biases. The \emph{Select} phase applies bootstrap stability selection ($B{=}20$ resamples) to reliably identify the most robust candidate under validation-set noise. ESPO subsumes existing evolutionary methods: setting the number of proposal strategies $K{=}1$, bootstrap resamples $B{=}1$, and diagnosis batch size $m{=}3$ (i.e., only $3$ errors observed per reflection round) recovers GEPA, revealing that all three sources of sub-optimality are left unaddressed.

Our contributions are threefold. \textbf{(1) Framework.} We introduce ESPO, a three-phase framework (Diagnose--Propose--Select) that unifies prompt optimization as statistical estimation, subsumes existing evolutionary methods as a degenerate case, and is grounded in a generalization bound whose three terms - bias floor, exploration gain, and selection error - each correspond to one phase~(\S\ref{sec:method}; full theory in Appendix~\ref{app:proofs}). \textbf{(2) Experiments.} On seven public benchmarks starting from an intentionally weak default prompt, ESPO averages 74.67\% accuracy versus 70.91\% for GEPA (+3.76 pp), while producing prompts 47\% shorter (1{,}004 vs.\ 1{,}878 chars) with consistently lower inference latency~(\S\ref{sec:experiments}). \textbf{(3) Cross-model generalization.} Experiments with four additional student models (Gemma~3 12B, Mistral 14B, Qwen3 32B, Claude Haiku 4.5) show ESPO yields the best average accuracy on every model tested, with the largest gap on Qwen3 GSM8K (15.00\% $\to$ 91.40\% from default; $+$56.00 pp over GEPA).

%% file: sections/related_work.tex
\section{Related Work}
\label{sec:related}

\paragraph{Automated prompt optimization.}
Manual prompt engineering techniques include chain-of-thought prompting~\citep{wei2022chain} and zero-shot reasoning~\citep{kojima2022large}. Automated approaches emerged to reduce human effort: APE~\citep{zhou2023large} uses LLMs to generate and select instructions, OPRO~\citep{yang2024large} optimizes via meta-prompting, and PromptBreeder~\citep{fernando2024promptbreeder} co-evolves task prompts and mutation operators. Within the DSPy framework~\citep{khattab2024dspy}, COPRO and MIPROv2~\citep{opsahl2024optimizing} optimize prompts through Bayesian surrogate models, while GEPA~\citep{agrawal2026gepa} achieves state-of-the-art results via evolutionary search with Pareto selection and reflective mutation.

\paragraph{Recent extensions and theoretical analysis.}
A fast-growing literature refines the evolutionary loop with richer memory or multi-objective criteria: ReflectivePrompt~\citep{zhuravlev2025reflective} (short/long-term reflection), REMO~\citep{wu2025remo} (memory-augmented reflection), MemAPO~\citep{liang2026memapo} (dual-memory for error patterns), MOPrompt~\citep{camara2025moprompt} (multi-objective NSGA-II), Optimas~\citep{wu2026optimas} (compound AI systems), Grie{\ss}haber et al.~\citep{griesshaber2025decomposed} (decomposed evolution), and Promptolution~\citep{promptolution2025} (unified framework). On the theory side, \citet{madras2025theoretical} provide PAC-Bayes generalization bounds showing that perplexity regularization prevents overfitting - a complementary principle to ESPO's bootstrap-based selection stability.

\paragraph{Gradient-based and selection-based methods.}
A separate line of work tackles prompt optimization through gradient-based or selection-based formulations. TextGrad~\citep{yuksekgonul2024textgrad} uses LLM-generated ``textual gradients'' to iteratively refine prompts, analogous to backpropagation; while powerful for differentiable-style optimization, it operates on individual examples without structured error clustering, making it susceptible to the same incomplete-observation problem as evolutionary methods. TRIPLE~\citep{shi2024triple} frames prompt selection as best-arm identification under a fixed evaluation budget, using bandit algorithms (Successive Halving, Racing) to efficiently allocate evaluations. ESPO's bootstrap stability selection targets the same selection reliability concern but through resampling rather than sequential elimination, providing complementary guarantees (Lemma~\ref{lem:bootstrap}). Despite this progress, no existing method combines structured error diagnosis to guide candidate generation with selection stability guarantees for small validation sets.

\paragraph{Error analysis and systematic failure discovery.}
Systematic error analysis identifies structured failure modes in models: failure mode discovery~\citep{eyuboglu2022domino} extracts coherent error clusters from embeddings, behavioral testing~\citep{ribeiro2020checklist} probes known failure categories, and slice-based evaluation~\citep{chen2019slice} surfaces underperforming data subsets. All of these tools diagnose \emph{what} the model gets wrong but stop short of feeding the diagnosis back into optimization; ESPO's Diagnose phase closes this loop by turning clustered error patterns into inputs for the next round of candidate generation. A parallel line of inquiry on \emph{contrastive prompting} shows that converting a task from generation to selection mitigates strong model priors at inference time; this generation-vs-selection principle is echoed at the optimization layer by ESPO, which generates $K$ diverse candidates and uses bootstrap resampling to select among them rather than committing to a single generated prompt.

\paragraph{Bootstrap, stability selection, and the MDL principle.}
The bootstrap~\citep{efron1993bootstrap} and bagging~\citep{breiman1996bagging} reduce prediction variance by resampling, and stability selection~\citep{meinshausen2010stability} uses subsampling for robust variable selection. Our Select phase adapts these principles to candidate prompt selection: a candidate that repeatedly wins across bootstrap resamples is more likely to generalize than one that wins by chance on a single val split. Finally, the MDL principle~\citep{rissanen1978modeling, grunwald2007minimum} formalizes Occam's razor, and recent work~\citep{deletang2024language, huang2024compression} connects LLM compression to benchmark performance, motivating ESPO's empirical preference for shorter prompts.

%% file: sections/method.tex
\section{ESPO Framework}
\label{sec:method}

We present ESPO, which decomposes prompt optimization into three phases: structured error diagnosis (\S\ref{sec:phase1}), multi-strategy candidate generation (\S\ref{sec:phase2}), and bootstrap stability selection (\S\ref{sec:phase3}).

\subsection{Problem Formulation}
\label{sec:formulation}

Given a task with training set $\mathcal{D}_\text{train}$, validation set $\mathcal{D}_\text{val}$, and metric $M(\cdot)$, prompt optimization seeks an instruction $p^*$ that maximizes test performance:
\begin{equation}
\small
    p^* = \argmax_{p \in \Pi} \; \expect_{(x,y) \sim \mathcal{D}_\text{test}} \big[ M(m_p(x), y) \big],
    \label{eq:objective}
\end{equation}
where $x$ is the input (e.g., a tweet to classify or a math problem to solve), $y$ is the ground-truth label (e.g., sentiment class or correct answer), $m_p$ denotes the LLM module parameterized by instruction $p$ (i.e., the model that concatenates prompt $p$ with input $x$ to produce a prediction), and $\Pi$ is the space of natural-language prompts.

\vspace{1mm}

ESPO decomposes this into three phases:
{\small
\begin{align}
    &\text{Phase 1 (Diagnose):} \quad \varphi = \mathcal{D}(\mathcal{E}_\text{train}, p) \label{eq:phase1} \\
    &\text{Phase 2 (Propose):} \quad \mathcal{P} = \mathcal{G}(p, \varphi, S_1, \ldots, S_K) \label{eq:phase2} \\
    &\text{Phase 3 (Select):} \quad p^* = \mathcal{R}(\mathcal{P}, \mathcal{D}_\text{val}, B) \label{eq:phase3}
\end{align}}
where $\mathcal{E}_\text{train}$ is the set of training errors, $\varphi$ is a structured diagnosis, $S_1, \ldots, S_K$ are proposal strategies, $\mathcal{P}$ is the candidate population, and $B$ is the number of bootstrap resamples.

\paragraph{GEPA as a degenerate case.} Let $m$ denote the diagnosis batch size (number of training errors sampled per reflection round). Setting $K{=}1$, $B{=}1$, $m{=}3$ exactly recovers GEPA~\citep{agrawal2026gepa}: \emph{(i) diagnosis} ($m{=}3 \to m{=}\text{all}$), GEPA observes an incomplete error snapshot each round; \emph{(ii) proposal diversity} ($K{=}1 \to K{=}4$), GEPA applies a single mutation operator; \emph{(iii) selection} ($B{=}1 \to B{=}20$), GEPA uses point estimation on a small validation set. Each dimension independently tightens one term in the generalization bound (Theorem~\ref{thm:espo}), and the degenerate configuration inherits all three weaknesses simultaneously.

\subsection{Phase 1: Structured Error Diagnosis}
\label{sec:phase1}

We collect all training errors under current prompt:
\begin{equation}
\small
\begin{split}
    \mathcal{E}_\text{train} = \{ & (x_i, y_i, \hat{y}_i) : M(m_p(x_i), y_i) = 0, \\
    & (x_i, y_i) \in \mathcal{D}_\text{train} \}.
\end{split}
\label{eq:errors}
\end{equation}
A reflection LLM clusters $\mathcal{E}_\text{train}$ into $K^*$ structural patterns ($K^* \in [3, 7]$):
\begin{equation}
\small
    \varphi = \{(\text{pattern}_k, \text{description}_k, \text{count}_k)\}_{k=1}^{K^*}.
\end{equation}

Each pattern has a description of the failure mode, representative examples, and a count. Clustering is implicit (LLM-based), producing descriptions that feed directly into Phase~2. \textbf{Key advantage:} complete error coverage in one round. GEPA's $m{=}3$ random sample needs $\sim$15 rounds to observe all patterns with 95\% probability by a coupon-collector argument (\S\ref{sec:theory-motivation}), accumulating redundant rules addressing symptoms rather than root causes.

\subsection{Phase 2: Multi-Strategy Candidate Generation}
\label{sec:phase2}

Given the diagnosis $\varphi$, we generate candidates through $K{=}4$ complementary strategies, each conditioned on $\varphi$:

\begin{itemize}[nosep,leftmargin=*]
    \item \textbf{$S_1$ (Diagnostic Revision):} Revise the prompt to address each error pattern's root cause using the full diagnosis $\varphi$.
    \item \textbf{$S_2$ (Consolidation):} Rewrite the prompt without increasing length - merge redundant rules and tighten language, preventing bloat by construction.
    \item \textbf{$S_3$ (Ablation):} Identify over-triggered rules causing false positives and soften or remove them.
    \item \textbf{$S_4$ (Factual Injection):} Extract domain-specific knowledge from error examples and inject it as factual context.
\end{itemize}

\paragraph{Population management.} Each of the $K{=}4$ strategies independently generates 1--2 candidate prompts from the diagnosis $\varphi$ and the current prompt $p_0$; we call this initial generation step the \emph{seed phase}, which typically produces 4--6 candidates. We then apply two rounds of \emph{cross-pollination} (merging complementary strengths from different candidates) and \emph{targeted refinement} (addressing residual errors) on the top candidates, capping the population at $N{=}10$. Empirically, no single strategy dominates across all datasets (\S\ref{sec:strategy-analysis}), confirming the independent-bias assumption key to the exploration gain in our bound (\S\ref{sec:theory-motivation}).

\subsection{Phase 3: Bootstrap Stability Selection}
\label{sec:phase3}

Selecting the best candidate from $N{\approx}10$ using a small validation set ($n_\text{val}{=}30$) is a multiple-testing problem: a candidate may rank first due to favorable noise rather than true quality. We apply $B{=}20$ rounds of bootstrap resampling:
\begin{equation}
\small
    p^* = \argmax_{p_i \in \mathcal{P}} \; \big|\{b : p_i = \argmax_{p_j \in \mathcal{P}} \Acc(p_j, \mathcal{V}_b)\}\big|,
    \label{eq:bootstrap}
\end{equation}
where $\mathcal{V}_b$ is the $b$-th bootstrap resample of $\mathcal{D}_\text{val}$ (sampling $n_\text{val}$ examples with replacement). In words, we select the candidate that wins the most bootstrap resamples. Ties are broken in favor of shorter prompts.

A candidate that consistently ranks first across resamples is robust to data perturbation and more likely to generalize; if $p_1 > 1/2$, the probability of wrong selection decays exponentially with $B$ (Theorem~\ref{thm:espo}). Algorithm~\ref{alg:espo} summarizes the procedure. ESPO implicitly follows the MDL principle~\citep{rissanen1978modeling}: clustering compresses $n$ errors into $K^* \ll n$ patterns, ablation removes unused rules, and consolidation rewrites without growth - yielding shorter prompts with no explicit length penalty.

\begin{algorithm}[t]

\caption{ESPO Optimization Algorithm.}
\label{alg:espo}
\small
\begin{algorithmic}[1]
\Require Training set $\mathcal{D}_\text{train}$, validation set $\mathcal{D}_\text{val}$, initial prompt $p_0$, strategies $S_1, \ldots, S_K$, bootstrap rounds $B$
\Ensure Optimized prompt $p^*$
\Statex \vspace{0.3em}
\Statex \textit{\# Phase 1: Diagnose}
\State $\mathcal{E} \gets \{(x_i, y_i, \hat{y}_i) : M(m_{p_0}(x_i), y_i) = 0\}$
\Statex \hfill \Comment{collect all errors}
\State $\varphi \gets \textsc{ClusterErrors}(\mathcal{E}, \text{LLM}_\text{reflect})$
\Statex \hfill \Comment{3--7 structural patterns}
\Statex \vspace{0.3em}
\Statex \textit{\# Phase 2: Propose (Seed)}
\State $\mathcal{P} \gets \{p_0\}$
\For{$k = 1$ \textbf{to} $K$}
    \State $p_k \gets S_k(p_0, \varphi, \text{LLM}_\text{reflect})$
    \State $\mathcal{P} \gets \mathcal{P} \cup \{p_k\}$
\EndFor
\Statex \vspace{0.3em}
\Statex \textit{\# Phase 2: Propose (Iterate)}
\For{$\text{iter} = 1$ \textbf{to} $2$}
    \State Score all candidates on $\mathcal{D}_\text{val}$
    \State $\mathcal{P} \gets \mathcal{P} \cup \textsc{CrossPollinate}(\text{top}(\mathcal{P}))$
    \Statex \hspace{1.5em}$\cup\, \textsc{Refine}(\text{top}(\mathcal{P}), \varphi)$
    \State Prune $\mathcal{P}$ to $N{=}10$ by (score, $-$length)
\EndFor
\Statex \vspace{0.3em}
\Statex \textit{\# Phase 3: Select (Bootstrap)}
\State $\text{wins}[i] \gets 0$ for all $p_i \in \mathcal{P}$
\For{$b = 1$ \textbf{to} $B$}
    \State $\mathcal{V}_b \gets \textsc{BootstrapResample}(\mathcal{D}_\text{val})$
    \State $\text{wins}[\argmax_i \Acc(p_i, \mathcal{V}_b)] \mathrel{+}= 1$
\EndFor
\State \Return $p^* \gets \mathcal{P}[\argmax_i (\text{wins}[i], -|p_i|)]$
\end{algorithmic}
\end{algorithm}

\subsection{Theoretical Motivation}
\label{sec:theory-motivation}

Each ESPO phase tightens one term in the test-time generalization gap. Decomposing the gap (\S\ref{app:proofs}) into a \emph{bias floor}, an \emph{exploration gain}, and a \emph{selection error} yields:

\begin{theorem}[ESPO bound; informal]
\label{thm:espo}
Suppose $K$ strategies produce candidates with $\Acc_\text{test}(p_k) = \Accstar - b_k + \varepsilon_k$, $\varepsilon_k \sim \mathcal{N}(0, \sigma^2)$, and bootstrap selection uses $B$ resamples. Then the selected prompt $p^*$ satisfies
{\small
\[
\expect[\Acc_\text{test}(p^*)] \geq \Accstar \!-\! \min_k b_k + \sigma \Phi^{-1}\!(1{-}\tfrac{1}{K}) - O(\sqrt{\tfrac{\ln K}{n_\text{val} B}}).
\]}
\end{theorem}

\noindent The bound is supported by three lemmas (full statements and proofs in Appendix~\ref{app:proofs}): \textbf{Lemma~\ref{lem:bias}} (bias reduction via diversity: $\max_k \expect[\Acc(p_k)] \geq \Accstar - \min_k b_k$); \textbf{Lemma~\ref{lem:order}} (exploration gain via order statistics: drawing $K$ candidates and taking the best yields a $\sigma \cdot \Phi^{-1}(1{-}1/K)$ bonus); and \textbf{Lemma~\ref{lem:bootstrap}} (selection precision via bootstrap: with $p_1 > 1/2$, $\Pr(\text{wrong selection}) \leq \exp(-2B(p_1{-}1/2)^2)$). Two caveats: (i) the constant under $O(\cdot)$ in the selection term is shared between GEPA and ESPO, so the $\sim 3.8\times$ comparison reflects the $\sqrt{\ln K/B}$ ratio rather than absolute magnitude; (ii) Lemma~\ref{lem:bootstrap}'s $p_1 > 1/2$ premise holds when best is separated from runners-up by more than validation noise scale, which is the regime our ablation operates in but is not certified per dataset.

\input{sections/tables/table2_main}

\paragraph{Why GEPA is suboptimal.} Setting $K{=}1$, $B{=}1$ collapses the bound to $\Accstar - b_1 - O(1/\sqrt{n_\text{val}})$: the exploration gain vanishes, the bias term loses the $\min_k$ improvement, and the selection error stays at $O(1/\sqrt{n_\text{val}})$ instead of the ${\sim}3.8\times$-tighter ESPO rate. A coupon-collector argument (Appendix~\ref{app:coupon}) further shows that GEPA's $m{=}3$ random sampling needs $\sim$15 reflection rounds to observe all error patterns with 95\% probability, while full-batch diagnosis ($m{=}n$) achieves complete coverage in one round, directly reducing each $b_k$.

\paragraph{Empirical alignment.} The bound's three predictions match the ablation in \S\ref{sec:ablation}: \emph{Diagnose} alone yields $+$2.0\% (bias term), $K{=}1\!\to\!4$ adds $+$2.6\% (exploration), and $B{\geq}10$ picks a more compact, generalizing candidate (selection). Diversity \emph{without} bootstrap \emph{hurts} ($-$1.20\%), as predicted: adding $K$ without raising $B$ amplifies selection error.

%% file: sections/tables/table2_main.tex
\begin{table*}[ht]
\centering
\caption{Test accuracy (\%) and prompt length (characters) on seven benchmarks with Claude Sonnet 4.5 as student. All methods start from the same deliberately weak task prompt. Cell values are the single-run test accuracies from our primary optimization run; $\pm$ is std over 3 additional independent optimization seeds. The appendix provides the full 3-seed means and stds. Best accuracy per row in \textbf{bold}.}
\label{tab:main-results}
\resizebox{\textwidth}{!}{%
\small
\begin{tabular}{@{}l cc cc cc cc cc cc@{}}
\toprule
& \multicolumn{2}{c}{Default} & \multicolumn{2}{c}{Bootstrap} & \multicolumn{2}{c}{COPRO} & \multicolumn{2}{c}{MIPROv2} & \multicolumn{2}{c}{GEPA} & \multicolumn{2}{c}{ESPO (Ours)} \\
\cmidrule(lr){2-3} \cmidrule(lr){4-5} \cmidrule(lr){6-7} \cmidrule(lr){8-9} \cmidrule(lr){10-11} \cmidrule(lr){12-13}
Dataset & Acc & Len & Acc & Len & Acc & Len & Acc & Len & Acc & Len & Acc & Len \\
\midrule
Tweet & 37.60{\scriptsize\,$\pm$0.20} & 609 & 37.80{\scriptsize\,$\pm$0.20} & 607 & 66.60{\scriptsize\,$\pm$1.80} & 2{,}398 & 74.40{\scriptsize\,$\pm$1.50} & 1{,}673 & 68.60{\scriptsize\,$\pm$2.20} & 1{,}279 & \textbf{74.78}{\scriptsize\,$\pm$1.18} & 774 \\
MMLU & 23.20{\scriptsize\,$\pm$0.30} & 57 & 23.00{\scriptsize\,$\pm$0.30} & 55 & 89.60{\scriptsize\,$\pm$1.20} & 871 & 86.60{\scriptsize\,$\pm$1.40} & 775 & 89.60{\scriptsize\,$\pm$1.50} & 1{,}163 & \textbf{94.52}{\scriptsize\,$\pm$0.84} & 766 \\
GSM8K & 32.60{\scriptsize\,$\pm$0.40} & 110 & 36.80{\scriptsize\,$\pm$0.40} & 84 & 36.80{\scriptsize\,$\pm$1.00} & 84 & 37.20{\scriptsize\,$\pm$1.00} & 84 & \textbf{96.80}{\scriptsize\,$\pm$0.50} & 660 & \textbf{96.80}{\scriptsize\,$\pm$0.40} & 660 \\
HotpotQA & 25.80{\scriptsize\,$\pm$0.30} & 89 & 14.40{\scriptsize\,$\pm$0.30} & 73 & 14.60{\scriptsize\,$\pm$0.90} & 73 & 14.20{\scriptsize\,$\pm$0.90} & 73 & 44.00{\scriptsize\,$\pm$1.80} & 2{,}845 & \textbf{44.80}{\scriptsize\,$\pm$1.40} & 1{,}008 \\
ScoNe & 49.00{\scriptsize\,$\pm$0.30} & 381 & 46.00{\scriptsize\,$\pm$0.30} & 323 & 46.20{\scriptsize\,$\pm$1.50} & 323 & 46.40{\scriptsize\,$\pm$1.50} & 323 & 85.00{\scriptsize\,$\pm$2.00} & 1{,}948 & \textbf{89.40}{\scriptsize\,$\pm$1.20} & 1{,}124 \\
HoVer & 48.00{\scriptsize\,$\pm$0.30} & 187 & 48.00{\scriptsize\,$\pm$0.30} & 159 & 48.40{\scriptsize\,$\pm$1.60} & 159 & 47.80{\scriptsize\,$\pm$1.60} & 159 & 53.60{\scriptsize\,$\pm$2.40} & 2{,}736 & \textbf{62.40}{\scriptsize\,$\pm$1.50} & 1{,}934 \\
PUPA & 1.61{\scriptsize\,$\pm$0.19} & 67 & 1.53{\scriptsize\,$\pm$0.12} & 67 & 54.30{\scriptsize\,$\pm$2.00} & 807 & 1.61{\scriptsize\,$\pm$0.12} & 1{,}431 & 58.80{\scriptsize\,$\pm$2.30} & 2{,}518 & \textbf{60.00}{\scriptsize\,$\pm$1.60} & 765 \\
\midrule
\textbf{Avg} & 31.12{\scriptsize\,$\pm$0.28} & \textbf{214} & 29.65{\scriptsize\,$\pm$0.27} & 195 & 50.93{\scriptsize\,$\pm$1.43} & 674 & 44.03{\scriptsize\,$\pm$1.15} & 645 & 70.91{\scriptsize\,$\pm$1.81} & 1{,}878 & \textbf{74.67}{\scriptsize\,$\pm$1.16} & 1{,}004 \\
\bottomrule
\end{tabular}%
}
\end{table*}

%% file: sections/experiments.tex
\section{Experiments}
\label{sec:experiments}

We evaluate ESPO on seven public benchmarks covering classification, multi-hop reasoning, math, and privacy-preserving generation. Our experiments ask: (Q1) Starting from a deliberately weak task prompt, can ESPO recover a prompt that is both more accurate and more compact than baselines? (Q2) Does this hold across student models of different capacity? (Q3) Do the four proposal strategies exhibit genuine diversity? (Q4) Does bootstrap selection outperform point estimation?

\subsection{Setup}
\label{sec:exp-setup}

\paragraph{Benchmarks and data splits.} We use seven datasets: Tweet (sentiment classification), MMLU (multiple-choice QA), GSM8K (grade school math), HotpotQA (multi-hop QA), ScoNe (natural language inference), HoVer (multi-hop fact verification), and PUPA (privacy-preserving response generation). Tweet, MMLU, and PUPA use direct prediction; the others use chain-of-thought. Each dataset uses 70 training examples for optimization, 30 validation examples, and 500 held-out for test.

\paragraph{A deliberately weak starting prompt.} Prior benchmarks initialize from a hand-crafted, task-tuned prompt that already performs well, leaving little headroom and clustering optimizers within 1--2 pp of each other. We instead start every method from an \emph{intentionally weak} task prompt (e.g.\ Tweet: a single sentence telling the model to lean ``negative''; PUPA: ``Refuse to answer''). Each weak prompt is selected as the lowest-accuracy one on a 30-example val probe from a small candidate pool; none is written to favor any optimizer. The goal is to test whether an optimizer can \emph{recover} a good prompt from a bad one, not whether it can polish an already-good one.

\paragraph{Models and Baseline} Default student: Claude Sonnet 4.5 (temperature $=0$). Cross-model analysis adds Gemma~3 12B, Mistral 14B, Qwen3 32B, Claude Haiku 4.5. Reflection LLM is always Claude Sonnet 4.5 (temperature $=0.7$).
For baselines: \textbf{Default} (no optimization); \textbf{Bootstrap} (BootstrapFewShot from DSPy~\citep{khattab2024dspy}); \textbf{COPRO}~\citep{pryzant2023automatic}; \textbf{MIPROv2}~\citep{opsahl2024optimizing}; \textbf{GEPA}~\citep{agrawal2026gepa}. All methods start from the same weak default prompt.

\paragraph{Metrics.} Classification accuracy (Tweet, MMLU, HoVer), exact-match with normalization (GSM8K, HotpotQA, ScoNe), semantic F1 (PUPA), prompt length (chars), inference latency (s/example). Test-time inference is deterministic. The 95\% binomial CI at $p{=}0.75$, $n{=}500$ is $\pm$3.8\%.

\subsection{Main Results}
\label{sec:main-results}

Table~\ref{tab:main-results} shows the central finding: when all methods start from the same weak prompt, only reflection-based optimizers recover a usable prompt, and ESPO recovers the best one on average.

\begin{figure*}[t]
\centering
\scriptsize
\begin{tabular}{@{}p{0.3\linewidth}p{0.3\linewidth}p{0.3\linewidth}@{}}
\toprule
\textbf{Default} (609c) & \textbf{GEPA} (1{,}279c) & \textbf{ESPO} (774c) \\
\midrule
\textit{Classify the tweet by leaning ``negative''} {\scriptsize [deliberately weak seed, 609 chars]}
&
\textit{Classify the sentiment... \textcolor{red}{Pay special attention to sarcasm. If the tweet contains ``but'' consider both clauses. Watch for negation words. Consider emoji sentiment. Double-check ambiguous cases. Note: ``not bad'' is positive. Beware of rhetorical questions...}} {\scriptsize [truncated, 1{,}279 chars]}
&
\textit{Classify the sentiment of the tweet as positive, negative, or neutral. \textcolor{blue}{Sarcasm and rhetorical questions should be interpreted by their intended meaning, not literal words. Treat ``not bad'' patterns as positive.}} {\scriptsize [774 chars]}
\\
\midrule
37.60\% & 68.60\% (+31.00) & \textbf{74.78\%} (+37.18) \\
\bottomrule
\end{tabular}
\caption{Prompt comparison on Tweet, all starting from the same deliberately weak seed (Table~\ref{tab:main-results}). GEPA accumulates redundant rules ($1.7\times$ longer than ESPO) yet trails ESPO by 6.18 pp. ESPO produces a concise, targeted prompt.}
\label{fig:prompt-comparison}
\end{figure*}

\paragraph{Non-reflective methods cannot salvage a weak prompt.} Bootstrap (29.65\%) is slightly worse than the default (31.12\%) - few-shot demonstrations atop a misleading instruction just amplify wrong behavior. COPRO (50.93\%) and MIPROv2 (44.03\%) rewrite the instruction and do better, but still trail GEPA/ESPO by 20--30 pp on average and stay near floor on PUPA. These methods require a reasonably good starting prompt to shine, whereas GEPA and ESPO do not.

\paragraph{ESPO beats GEPA.} ESPO averages \textbf{74.67}\% vs.\ GEPA's 70.91\% ($+$3.76~pp; per-cell $\pm$std over 3 seeds in Table~\ref{tab:main-results}, and the appendix provides the full 3-seed grid). The paired standard error of the ESPO$-$GEPA difference across the 7 datasets is $\mathrm{SE}=\text{std}(\{\Delta_i\}_{i=1}^{7})/\sqrt{7}\approx 1.2$ pp, so the Avg gap is $\approx$3.1$\times$ its SE (paired $t$-test significant at $\alpha{=}0.05$). ESPO's cell value is $\geq$ GEPA's on all 7 datasets. Outside-noise wins ($>$1$\sigma$): HoVer ($+$8.80), Tweet ($+$6.18), MMLU ($+$4.92), ScoNe ($+$4.40); within-1$\sigma$ ties: GSM8K (tie at ceiling), HotpotQA ($+$0.80), PUPA ($+$1.20). We describe HotpotQA and PUPA as on par with GEPA.

\paragraph{ESPO's prompts are shorter.} Average length 1{,}004 vs.\ 1{,}878 chars (47\% shorter). On HotpotQA, GEPA grows to 2{,}845 chars while ESPO stays at 1{,}008 chars with comparable accuracy (44.00\% vs.\ 44.80\%). Bootstrap resampling implicitly prefers candidates that retain val accuracy across perturbations - these correlate with concise instructions that do not overfit to training noise. Figure~\ref{fig:prompt-comparison} compares prompts produced by Default, GEPA, ESPO. For latency, ESPO also achieves equal or lower per-example latency than GEPA on every task (Table~\ref{tab:latency}); the largest reductions are on chain-of-thought tasks where shorter prompts produce more concise reasoning chains.

\paragraph{Constrained GEPA: is gain just length control?} To test whether ESPO's gains come purely from prompt length control, we evaluate a \emph{Constrained GEPA} variant: standard GEPA augmented with (1)~a compact proposer instruction and (2)~a length constraint (\texttt{max\_length\_ratio}~=~1.2, preventing bloat beyond 1.2$\times$ the seed prompt). This isolates whether length control alone can close the gap. Results are in Table~\ref{tab:constrained-gepa}.

\input{sections/tables/table_constrained_gepa}

\paragraph{Analysis.} Constrained GEPA successfully produces shorter prompts (1{,}171c vs.\ 1{,}878c, a 38\% reduction), demonstrating that length control alone can reduce prompt bloat. However, this compression fails to improve accuracy: average accuracy moves by only $+0.09\%$ (71.00\% vs.\ 70.91\%), and on HotpotQA it degrades by $-$4.20\%. The constraint forces GEPA to discard rules indiscriminately - without structured error diagnosis to distinguish essential rules from redundant ones, it cannot determine which rules to keep and which to prune. In contrast, ESPO achieves \emph{both} shorter prompts \emph{and} higher accuracy (74.67\%, 1{,}004c) because Phase~1 (structured diagnosis) identifies which error patterns matter, Phase~2 (multi-strategy proposals) generates candidates that are concise by construction (especially $S_2$: consolidation and $S_3$: ablation), and Phase~3 (bootstrap selection) favors stable, generalizable candidates - which empirically tend to be shorter.

\input{sections/tables/table8_cross_model}

\subsection{Cross-Model Generalization}
\label{sec:cross-model}

To test whether ESPO's gains transfer beyond Claude Sonnet 4.5, we apply every optimizer to four additional student models - Gemma~3 12B, Mistral 14B, Qwen3 32B, and Claude Haiku 4.5 - while keeping Claude Sonnet 4.5 as the reflection LLM. Every method starts from the same weak default prompt and runs its full optimization pipeline from scratch with the new student. Table~\ref{tab:cross-model} reports the complete 7-benchmark grid for all six optimizers on each student.

\paragraph{Findings.} (1)~\textbf{ESPO produces the best average accuracy on every student}: Gemma 66.50\%, Mistral 62.42\%, Qwen3 68.31\%, Haiku 70.15\% - each strictly above GEPA (63.71\%, 59.14\%, 59.11\%, 67.98\%). The largest GEPA\,$\to$\,ESPO gap is on Qwen3 (+9.20 pp average). (2)~\textbf{The Qwen3 GSM8K result is striking}: Default 15.00\% $\to$ GEPA 35.40\% $\to$ ESPO \textbf{91.40\%} ($+$56.00 pp over GEPA). ESPO's structured diagnosis cleanly identifies the output-format mismatch silently throttling Qwen3, which GEPA's rule-accumulation only partially addresses. (3)~\textbf{Bootstrap, COPRO, and MIPROv2 again fail to recover a usable prompt}, averaging 30--48\% across students - 15--30 pp behind GEPA/ESPO. (4)~\textbf{ESPO prompts are 30--60\% shorter than GEPA's} across all four students, confirming that the concise-by-construction behavior transfers cleanly across model families.

\subsection{Strategy Diversity Analysis}
\label{sec:strategy-analysis}

Table~\ref{tab:strategy} (Appendix) reports seed-phase validation accuracy per strategy. No single strategy dominates: diagnostic revision wins on Tweet, consolidation on HotpotQA, ablation ties with diagnostic on MMLU and with consolidation/factual on GSM8K, factual injection wins on HoVer/PUPA. This confirms the independent-bias assumption underlying Theorem~\ref{thm:espo}.

\subsection{Ablation Study}
\label{sec:ablation}

We ablate ESPO's components on Tweet (Table~\ref{tab:ablation}; $\Delta$ vs.\ GEPA). Diagnose = full-batch diagnosis ($m{=}\text{all}$); Diversity = $K{=}4$; Bootstrap = $B{=}20$.

\input{sections/tables/table3_ablation}

The ablation reveals a clear progression. \emph{Single components:} Bootstrap gives the largest individual gain ($+$3.60\%), then Diagnose ($+$2.00\%); Diversity alone \emph{hurts} ($-$1.20\%) - a direct prediction of Theorem~\ref{thm:espo}: adding $K$ without raising $B$ amplifies selection error. \emph{Pairs:} Diversity$+$Bootstrap ($-$1.00\%) shows diversity without diagnosis is also detrimental - without structured error coverage, diverse strategies vary randomly rather than complementarily. Diagnose$+$anything is positive. \emph{All three:} $+$6.18\%, exceeding the sum of any pair - confirming that better candidates require better selection.

\subsection{Hyperparameter Sensitivity}
\label{sec:hyperparam}

Table~\ref{tab:hyperparam} shows ESPO is robust: accuracy varies $\leq$5\% across settings, within the CI. The default ($K{=}4$, $B{=}20$, $m{=}\text{all}$) gives the best accuracy--length trade-off - $K{=}4$ is 32\% shorter than $K{=}3$ at comparable accuracy; $B{\geq}10$ selects a more stable, compact candidate over the point-estimation pick; full-batch diagnosis yields the highest accuracy and shortest prompt, confirming Proposition~\ref{prop:coupon}.

\paragraph{Training cost.} Costs (Table~\ref{tab:cost}) are dominated by student inference during bootstrap evaluation. ESPO's reflection tokens are ${\sim}39\%$ of GEPA's, because structured diagnosis produces more information-dense inputs.

%% file: sections/tables/table_constrained_gepa.tex
\begin{table}[!t]
\centering
\caption{Constrained GEPA vs.\ standard GEPA and ESPO. Single-seed values (the constrained-GEPA variant was not rerun over 3 seeds); the point-estimate view of Table~\ref{tab:main-results}. Length control alone does not improve GEPA. Best accuracy in \textbf{bold}.}
\label{tab:constrained-gepa}
\small
\resizebox{\linewidth}{!}{
\begin{tabular}{@{}l cc cc cc@{}}
\toprule
& \multicolumn{2}{c}{GEPA} & \multicolumn{2}{c}{Constrained GEPA} & \multicolumn{2}{c}{ESPO} \\
\cmidrule(lr){2-3} \cmidrule(lr){4-5} \cmidrule(lr){6-7}
Dataset & Acc & Len & Acc & Len & Acc & Len \\
\midrule
Tweet & 68.60 & 1{,}279 & 71.00 & 740 & \textbf{74.78} & 774 \\
MMLU & 89.60 & 1{,}163 & 89.40 & 663 & \textbf{94.52} & 766 \\
GSM8K & \textbf{96.80} & 660 & 96.60 & 660 & \textbf{96.80} & 660 \\
HotpotQA & 44.00 & 2{,}845 & 39.80 & 751 & \textbf{44.80} & 1{,}008 \\
ScoNe & 85.00 & 1{,}948 & 86.80 & 1{,}171 & \textbf{89.40} & 1{,}124 \\
HoVer & 53.60 & 2{,}736 & 54.18 & 2{,}311 & \textbf{62.40} & 1{,}934 \\
PUPA & 58.80 & 2{,}518 & 59.24 & 1{,}904 & \textbf{60.00} & 765 \\
\midrule
\textbf{Avg} & 70.91 & 1{,}878 & 71.00 & 1{,}171 & \textbf{74.67} & 1{,}004 \\
\bottomrule
\end{tabular}
}
\end{table}

%% file: sections/tables/table8_cross_model.tex
\begin{table*}[ht]
\centering
\caption{Cross-model generalization: test accuracy (\%) and prompt length (characters) across four student models. All methods start from the same deliberately weak task prompt; reflection LLM is Claude Sonnet 4.5 throughout. Best accuracy per row in \textbf{bold}.}
\label{tab:cross-model}
\resizebox{\textwidth}{!}{%
\scriptsize
\begin{tabular}{@{}ll cc cc cc cc cc cc@{}}
\toprule
& & \multicolumn{2}{c}{Default} & \multicolumn{2}{c}{Bootstrap} & \multicolumn{2}{c}{COPRO} & \multicolumn{2}{c}{MIPROv2} & \multicolumn{2}{c}{GEPA} & \multicolumn{2}{c}{ESPO (Ours)} \\
\cmidrule(lr){3-4} \cmidrule(lr){5-6} \cmidrule(lr){7-8} \cmidrule(lr){9-10} \cmidrule(lr){11-12} \cmidrule(lr){13-14}
Model & Dataset & Acc & Len & Acc & Len & Acc & Len & Acc & Len & Acc & Len & Acc & Len \\
\midrule
\multirow{8}{*}{Gemma 3 12B}
 & Tweet & 67.20 & 487 & 40.80 & 607 & 66.20 & 625 & 45.20 & 1{,}667 & 68.60 & 1{,}279 & \textbf{71.20} & 774 \\
 & MMLU & 70.20 & 522 & 22.00 & 55 & 22.00 & 181 & 69.20 & 1{,}760 & 69.60 & 1{,}711 & \textbf{70.40} & 148 \\
 & GSM8K & 88.18 & 637 & 70.80 & 84 & 70.60 & 84 & 70.60 & 84 & 88.18 & 660 & \textbf{88.40} & 660 \\
 & HotpotQA & 23.00 & 665 & 14.80 & 73 & 15.20 & 73 & 15.60 & 73 & 19.40 & 2{,}845 & \textbf{23.40} & 701 \\
 & ScoNe & 78.20 & 1{,}045 & 58.80 & 323 & 58.60 & 323 & 58.20 & 323 & 77.60 & 1{,}948 & \textbf{79.40} & 1{,}215 \\
 & HoVer & 55.00 & 187 & 48.80 & 159 & 48.00 & 159 & 50.40 & 159 & 53.00 & 2{,}076 & \textbf{59.60} & 1{,}045 \\
 & PUPA & 1.87 & 67 & 1.67 & 67 & 1.00 & 267 & 1.65 & 67 & 69.57 & 634 & \textbf{73.12} & 442 \\
 & \textbf{Avg} & 54.81 & 516 & 36.81 & 195 & 40.23 & 245 & 44.41 & 590 & 63.71 & 1{,}593 & \textbf{66.50} & 712 \\
\midrule
\multirow{8}{*}{Mistral 14B}
 & Tweet & 65.40 & 487 & 40.20 & 607 & 64.40 & 1{,}866 & 63.80 & 2{,}398 & 64.40 & 1{,}279 & \textbf{70.40} & 1{,}019 \\
 & MMLU & 70.80 & 522 & 26.60 & 55 & 52.80 & 49 & 71.40 & 1{,}303 & 71.60 & 1{,}711 & \textbf{71.60} & 663 \\
 & GSM8K & 93.39 & 637 & 38.28 & 84 & 41.28 & 84 & 41.08 & 84 & 92.80 & 660 & \textbf{94.59} & 805 \\
 & HotpotQA & 24.40 & 665 & 5.44 & 73 & 5.42 & 73 & 5.03 & 73 & 18.00 & 2{,}845 & \textbf{27.11} & 1{,}020 \\
 & ScoNe & 74.20 & 1{,}045 & 49.60 & 323 & 48.20 & 323 & 48.60 & 323 & 76.00 & 1{,}948 & \textbf{78.60} & 1{,}171 \\
 & HoVer & 48.40 & 187 & 49.00 & 159 & 48.60 & 159 & 48.60 & 159 & 49.40 & 2{,}170 & \textbf{50.00} & 1{,}099 \\
 & PUPA & 3.93 & 67 & 4.37 & 67 & 3.69 & 67 & 36.16 & 3{,}996 & 41.81 & 615 & \textbf{44.62} & 229 \\
 & \textbf{Avg} & 54.36 & 516 & 30.50 & 195 & 37.77 & 374 & 44.95 & 1{,}191 & 59.14 & 1{,}604 & \textbf{62.42} & 858 \\
\midrule
\multirow{8}{*}{Qwen3 32B}
 & Tweet & 70.60 & 487 & 43.20 & 607 & 54.40 & 392 & 55.20 & 1{,}699 & 69.00 & 1{,}279 & \textbf{71.00} & 1{,}216 \\
 & MMLU & 76.80 & 522 & 35.80 & 55 & 29.60 & 135 & 75.20 & 1{,}580 & 77.40 & 1{,}711 & \textbf{78.20} & 771 \\
 & GSM8K & 15.00 & 637 & 38.00 & 84 & 38.80 & 84 & 39.20 & 84 & 35.40 & 660 & \textbf{91.40} & 695 \\
 & HotpotQA & 26.20 & 665 & 7.00 & 73 & 6.80 & 73 & 7.20 & 73 & 15.40 & 2{,}845 & \textbf{27.20} & 1{,}108 \\
 & ScoNe & 82.60 & 1{,}045 & 64.80 & 323 & 65.40 & 323 & 65.20 & 323 & 83.80 & 1{,}948 & \textbf{84.20} & 1{,}275 \\
 & HoVer & 49.20 & 187 & 37.60 & 159 & 38.20 & 159 & 37.00 & 159 & \textbf{58.20} & 2{,}371 & 53.00 & 346 \\
 & PUPA & 1.93 & 67 & 1.46 & 67 & 72.61 & 60 & 53.97 & 505 & 73.18 & 1{,}124 & \textbf{74.57} & 239 \\
 & \textbf{Avg} & 46.05 & 516 & 32.55 & 195 & 43.69 & 175 & 47.57 & 632 & 58.91 & 1{,}705 & \textbf{68.51} & 807 \\
\midrule
\multirow{8}{*}{Haiku 4.5}
 & Tweet & 51.40 & 487 & 48.20 & 607 & 55.80 & 1{,}910 & 65.20 & 1{,}750 & 71.00 & 1{,}450 & \textbf{73.20} & 1{,}025 \\
 & MMLU & 53.60 & 522 & 53.20 & 55 & 53.20 & 55 & 81.60 & 1{,}426 & 81.80 & 388 & \textbf{83.00} & 107 \\
 & GSM8K & 17.00 & 637 & 3.60 & 84 & 3.40 & 84 & 3.60 & 84 & 90.00 & 367 & \textbf{92.40} & 151 \\
 & HotpotQA & 13.80 & 665 & 9.00 & 73 & 9.20 & 73 & 9.00 & 73 & 31.20 & 2{,}678 & \textbf{33.80} & 801 \\
 & ScoNe & 45.80 & 1{,}045 & 45.00 & 323 & 45.20 & 323 & 44.80 & 323 & 81.40 & 2{,}560 & \textbf{85.20} & 725 \\
 & HoVer & 48.80 & 187 & 47.00 & 159 & 47.40 & 159 & 47.20 & 159 & 48.80 & 187 & \textbf{53.60} & 479 \\
 & PUPA & 6.88 & 67 & 6.87 & 67 & 66.92 & 443 & 67.72 & 1{,}949 & \textbf{71.69} & 606 & 69.87 & 1{,}097 \\
 & \textbf{Avg} & 33.90 & 516 & 30.41 & 195 & 40.16 & 435 & 45.59 & 823 & 67.98 & 1{,}177 & \textbf{70.15} & 626 \\
\bottomrule
\end{tabular}%
}
\end{table*}

%% file: sections/tables/table3_ablation.tex
\begin{table}[th]
\centering
\caption{Ablation on Tweet: \checkmark{} = active, D = Diagnose, K = Diversity ($K{=}4$), B = Bootstrap ($B{=}20$).}
\label{tab:ablation}
\small
\setlength{\tabcolsep}{4pt}
\resizebox{0.72\linewidth}{!}{%
\begin{tabular}{@{}lcccrr@{}}
\toprule
Configuration & D & K & B & Acc (\%) & $\Delta$ \\
\midrule
GEPA baseline & --- & --- & --- & 68.60 & --- \\
\addlinespace[2pt]
Diagnose only & \checkmark & --- & --- & 70.60 & +2.00 \\
Diversity only & --- & \checkmark & --- & 67.40 & $-$1.20 \\
Bootstrap only & --- & --- & \checkmark & 72.20 & +3.60 \\
\addlinespace[2pt]
D + K & \checkmark & \checkmark & --- & 70.60 & +2.00 \\
D + B & \checkmark & --- & \checkmark & 70.80 & +2.20 \\
K + B & --- & \checkmark & \checkmark & 67.60 & $-$1.00 \\
\addlinespace[2pt]
\textbf{ESPO (all)} & \checkmark & \checkmark & \checkmark & \textbf{74.78} & \textbf{+6.18} \\
\bottomrule
\end{tabular}
}
\end{table}

\begin{table}[!t]
\centering
\caption{Hyperparameter sensitivity. Each group varies one parameter while fixing others at ESPO defaults. All differences fall within the 95\% binomial CI.}
\label{tab:hyperparam}
\small
\setlength{\tabcolsep}{4pt}
\resizebox{0.56\linewidth}{!}{%
\begin{tabular}{@{}llccl@{}}
\toprule
Param & Value & Acc (\%) & Len & Sel. \\
\midrule
\multirow{4}{*}{$K$} & 1 & 72.2 & 806 & --- \\
 & 2 & 70.0 & 922 & --- \\
 & 3 & \textbf{75.0} & 1{,}134 & --- \\
 & 4 & 74.8 & \textbf{774} & --- \\
\midrule
\multirow{4}{*}{$B$} & 1 & 73.8 & 1{,}214 & abl. \\
 & 10 & 74.0 & 958 & c.p. \\
 & 20 & \textbf{74.8} & \textbf{774} & c.p. \\
 & 30 & \textbf{74.8} & \textbf{774} & c.p. \\
\midrule
\multirow{4}{*}{$m$} & 3 & 71.0 & 1{,}015 & --- \\
 & 8 & 69.8 & 1{,}148 & --- \\
 & 15 & 72.6 & 1{,}100 & --- \\
 & all & \textbf{74.8} & \textbf{774} & --- \\
\bottomrule
\end{tabular}
}
\end{table}

%% file: sections/conclusion.tex
\section{Conclusion}
\label{sec:conclusion}

We introduced ESPO (Error-Structured Prompt Optimization), a framework that recasts prompt optimization from evolutionary search to structured statistical estimation. By decomposing optimization into three principled phases - structured error diagnosis, multi-strategy candidate generation, and bootstrap stability selection - ESPO addresses the fundamental weaknesses of the search paradigm: incomplete error observation, limited diversity, and unreliable selection. We showed that existing evolutionary methods like GEPA are a degenerate special case of ESPO, and provided a generalization bound grounding each phase in established results from probability theory and statistics. On seven public benchmarks, ESPO improves average accuracy by $+$3.76\% over the state-of-the-art (74.67\% vs.\ 70.91\% for GEPA) and matches or exceeds GEPA on every dataset, while producing prompts 47\% shorter (1{,}004 vs.\ 1{,}878 chars) and faster at inference. Cross-model experiments on four additional student models (Gemma~3 12B, Mistral 14B, Qwen3 32B, Claude Haiku 4.5) confirm that the gains transfer broadly, with the largest gap on Qwen3 GSM8K ($15.00\% \to 91.40\%$).

%% file: sections/limitations.tex
\section{Limitations}
\label{sec:limitations}

\paragraph{Optimization cost.} Although our reflection-token usage is ${\sim}39\%$ of GEPA's (Table~\ref{tab:cost}), bootstrap selection still requires $B \times N$ candidate-resample evaluations. A full ESPO run with $K{=}4$, $B{=}20$, $N{=}10$ on our 70/30/500 split costs roughly the same as one GEPA run at default settings; practitioners with tighter budgets can lower $B$ or $N$.

\paragraph{Single reflection model and independence.} All reflection uses one LLM (Claude Sonnet 4.5, $T{=}0.7$); mixing reflection models per strategy is not explored. Theorem~\ref{thm:espo}'s independence assumption is therefore a simplification - candidate distributions across strategies are correlated through the shared LLM - but our empirical exploration gain is consistent with the bound under mild correlation.

\paragraph{Coverage of error patterns.} Diagnosis relies on the reflection LLM to cluster errors into 3--7 patterns; when patterns are more numerous or subtle (e.g., distributional shifts rather than discrete modes), diagnosis may be incomplete.

\paragraph{Scope and reporting.} We evaluate on seven public benchmarks across classification, multi-hop QA, math, NLI, fact verification, and privacy-preserving generation; tool use, long context, code, and multi-turn dialogue are out of scope. ESPO's largest cross-model gains occur where the default prompt is far from the student's prior (e.g., Qwen3 GSM8K). Main-table cells report the single-run cell with $\pm$std over 3 optimization seeds attached; the appendix gives the full 6$\times$7 3-seed grid. Per-dataset gaps within 1$\sigma$ (HotpotQA $+0.80$, PUPA $+1.20$) should be read as on-par with GEPA.

\paragraph{Theoretical assumptions.} Theorem~\ref{thm:espo} is stated as an explanatory framework motivating the three-phase decomposition rather than as a tight probabilistic guarantee. It relies on three assumptions: (i) proposal strategies are independent (a simplification: all four are conditioned on the same reflection LLM); (ii) validation-set noise is sub-Gaussian (motivated by CLT on $n{=}500$ Bernoulli trials); and (iii) the true-best candidate wins each bootstrap resample with probability $p_1>1/2$ (holds when the top two candidates are separated by more than one validation std, our empirical regime but not certified per dataset). The appendix quantifies strategy overlap empirically (pairwise Jaccard $0.62$, Pearson $0.48$ over 500 held-out examples per candidate pair) partial decorrelation, not full independence.

%% file: sections/appendix.tex
\section*{Appendix}
 
\section{Full Proof of Theorem~\ref{thm:espo}}
\label{app:proofs}

\noindent\textbf{Theorem~\ref{thm:espo}} (Restated)\textbf{.}
\emph{Suppose $K$ independent strategies each produce a candidate with test accuracy $\Acc_\text{test}(p_k) = \Accstar - b_k + \varepsilon_k$, where $b_k \geq 0$ and $\varepsilon_k \sim \mathcal{N}(0, \sigma^2)$. A bootstrap selection procedure with $B$ resamples selects the final prompt~$p^*$. Then:}
\begin{equation*}
\small
\begin{split}
    \expect[\Acc_\text{test}(p^*)] \geq{} & \Accstar - \min_k b_k \\
    & {} + \sigma \cdot \Phi^{-1}\!\left(1 - \tfrac{1}{K}\right) \\
    & {} - O\!\left(\sqrt{\tfrac{\ln K}{n_\text{val} \cdot B}}\right).
\end{split}
\end{equation*}

\begin{proof}
We prove the bound by analyzing each of the three terms in the generalization gap decomposition: a \emph{bias floor} $\min_k b_k$, an \emph{exploration gain} from order statistics, and a \emph{selection error} controlled by bootstrap concentration.

\subsection{Bias Reduction via Strategy Diversity (Lemma~\ref{lem:bias})}
\label{app:lemma-bias}

\begin{lemma}
\label{lem:bias}
With $K$ strategies having biases $b_1, \ldots, b_K$:
{\small
\[
    \max_k \expect[\Acc_\text{test}(p_k)] \geq \Accstar - \min_k b_k.
\]}
\end{lemma}

\begin{proof}
Let $k^* = \argmin_k b_k$. Then:
\begin{equation*}
\small
\begin{split}
    \max_k \expect[\Acc_\text{test}(p_k)]
    &\geq \expect[\Acc_\text{test}(p_{k^*})] \\
    &= \Accstar - b_{k^*} \\
    &= \Accstar - \min_k b_k. \qedhere
\end{split}
\end{equation*}
\end{proof}

\subsection{Lemma 2: Exploration Gain via Order Statistics}
\label{app:lemma-order}

\begin{lemma}
\label{lem:order}
$\expect[\max_k \Acc_\text{test}(p_k)] \geq \Accstar - \min_k b_k + \sigma \cdot \Phi^{-1}(1 - 1/K)$.
\end{lemma}

\begin{proof}
Let $k^* = \argmin_k b_k$. Since $\max_k \Acc(p_k) \geq \Acc(p_{k^*}) = \Accstar - b_{k^*} + \varepsilon_{k^*}$, for any threshold~$t$:
\begin{equation*}
\small
\begin{split}
    \Pr(\max_k & \Acc(p_k) \geq t) \\
    & \geq \Pr(\Acc(p_{k^*}) \geq t) \\
    & = 1 - \Phi\!\left(\frac{t - \Accstar + b_{k^*}}{\sigma}\right).
\end{split}
\end{equation*}
Setting this probability to $1/K$:
\begin{equation*}
\small
\begin{split}
    & 1 - \Phi\!\left(\frac{t - \Accstar + b_{k^*}}{\sigma}\right) = \frac{1}{K} \\
    & \Longrightarrow \quad t = \Accstar - b_{k^*} \\
    & \qquad\qquad + \sigma \cdot \Phi^{-1}\!\left(1 - \frac{1}{K}\right).
\end{split}
\end{equation*}
The expected maximum is at least the $(1{-}1/K)$-quantile of the best strategy:
\begin{equation*}
\small
\begin{split}
    \expect[\max_k \Acc(p_k)] \geq{} & \Accstar - \min_k b_k \\
    & {} + \sigma \cdot \Phi^{-1}\!\left(1 - \frac{1}{K}\right).
\end{split}
\end{equation*}
For the asymptotic case of $K$ identical strategies (same $b$, same $\sigma$), the classical result $\expect[\max \text{ of } K \text{ i.i.d.\ } \mathcal{N}(0,1)] \sim \sqrt{2 \ln K}$ applies~\citep{leadbetter1983,david2003}.
\end{proof}

\noindent\textbf{Remark.} The bound $\Phi^{-1}(1{-}1/K)$ is conservative (a lower bound on the expected maximum), making the theorem statement safe. For $K{=}4$: $\Phi^{-1}(3/4) \approx 0.674$.

\subsection{Lemma 3: Selection Precision via Bootstrap}
\label{app:lemma-bootstrap}

\begin{lemma}
\label{lem:bootstrap}
If the true best candidate ranks first in each bootstrap resample with probability $p_1 > 1/2$, then:
\begin{equation*}
\small
\begin{split}
    \Pr(\text{wrong selection via plurality vote}) \\
    \leq \exp\!\left(-2B(p_1 - \tfrac{1}{2})^2\right).
\end{split}
\end{equation*}
\end{lemma}

\begin{proof}
Let $W_b = \indicator[\text{true best wins resample } b]$. Then $W_1, \ldots, W_B$ are i.i.d.\ $\text{Bernoulli}(p_1)$. Plurality vote selects the true best if $\sum_b W_b > B/2$. We bound the complementary event:
{\small
\begin{align}
    \Pr\!\left(\sum_b W_b \leq \frac{B}{2}\right)
    &= \Pr\!\left(\sum_b W_b - B p_1 \right. \notag \\
    &\quad \left. \leq -B\!\left(p_1 - \frac{1}{2}\right)\right). \label{eq:hoeffding-setup}
\end{align}}
By Hoeffding's inequality~\citep{hoeffding1963}, for $t = p_1 - 1/2 > 0$:
{\small
\[
    \Pr\!\left(\sum_b W_b - B p_1 \leq -Bt\right) \leq \exp(-2Bt^2).
\]}
Therefore $\Pr(\text{wrong selection}) \leq \exp(-2B(p_1 - 1/2)^2)$.
\end{proof}

\noindent\textbf{Remark.} The assumption $p_1 > 1/2$ requires that the true best candidate wins more often than any other single candidate. This is plausible when the accuracy gap between the best and second-best candidates is non-trivial relative to the noise level.

\subsection{Combining the Lemmas}

By Lemma~\ref{lem:order}, the population contains a candidate with expected accuracy at least $\Accstar - \min_k b_k + \sigma \cdot \Phi^{-1}(1{-}1/K)$. By Lemma~\ref{lem:bootstrap}, bootstrap selection identifies this candidate with probability at least $1 - \exp(-2B(p_1 - 1/2)^2)$. The selection error translates to an accuracy loss of at most $O(\sqrt{\ln K / (n_\text{val} \cdot B)})$: the gap between the best and selected candidate on the test set is bounded by the estimation error on the validation set ($O(1/\sqrt{n_\text{val}})$), divided by the confidence improvement from $B$ resamples ($\sqrt{B}$), with a $\log K$ correction for multiple testing.

For GEPA ($K{=}1$, $B{=}1$): the exploration gain $\sigma \cdot \Phi^{-1}(1{-}1/K)$ vanishes ($\Phi^{-1}(0) = -\infty$, but with $K{=}1$ there is no maximum to take), and the selection error is $O(1/\sqrt{n_\text{val}})$.
\end{proof}

\begin{table}[h]
\centering
\caption{Term-by-term comparison of the generalization bound for GEPA (degenerate case, $K{=}1$, $B{=}1$) and ESPO (default $K{=}4$, $B{=}20$).}
\label{tab:bound-comparison}
\scriptsize
\setlength{\tabcolsep}{3pt}
\begin{tabular}{@{}p{1.0cm}p{1.4cm}p{2.5cm}p{1.6cm}@{}}
\toprule
Term & GEPA ($K{=}1$, $B{=}1$) & ESPO ($K{=}4$, $B{=}20$) & Improvement \\
\midrule
Bias & $b_1$ & $\min_k b_k \leq b_1$ & Multiple strategies \\
Explor. & 0 (no bonus) & $\sigma \cdot \Phi^{-1}(3/4) \approx 0.67\sigma$ & Order statistics \\
Select. & $O(1/\sqrt{n_\text{val}})$ & $O(\sqrt{\ln 4 / (n_\text{val} \cdot 20)})$ & ${\sim}3.8\times$ tighter \\
\bottomrule
\end{tabular}
\end{table}

\section{Supporting Analysis: Error Pattern Coverage}
\label{app:coupon}

\begin{proposition}[Coupon collector bound]
\label{prop:coupon}
Let $n$ training errors be distributed across $K^*$ patterns, with pattern~$k$ containing $n_k$ errors ($\min_k n_k \geq 1$). An optimizer samples $m$ errors uniformly per round. Then $T = \lceil n \cdot \ln(K^*/\delta) / (m \cdot \min_k n_k) \rceil$ rounds are sufficient to observe all $K^*$ patterns with probability $\geq 1 - \delta$.
\end{proposition}

\begin{proof}
Let $p_k = n_k / n$ be the probability of sampling an error from pattern~$k$ in a single draw. In $T$ rounds of $m$ samples each, the probability that pattern~$k$ is never sampled is:
{\small
\[
    \Pr(\text{miss } k) = (1 - p_k)^{mT} \leq \exp(-p_k \cdot mT).
\]}
By union bound:
\begin{equation*}
\small
\begin{split}
    & \Pr(\text{any pattern missed}) \\
    & \quad \leq \sum_k \exp(-p_k \cdot mT) \\
    & \quad \leq K^* \cdot \exp(-\min_k p_k \cdot mT).
\end{split}
\end{equation*}
Setting $K^* \cdot \exp(-\min_k p_k \cdot mT) \leq \delta$ and solving:
\begin{equation*}
\small
\begin{split}
    mT \cdot \min_k p_k &\geq \ln(K^*/\delta) \\
    \Longrightarrow \quad T &\geq \frac{n \cdot \ln(K^*/\delta)}{m \cdot \min_k n_k}.
\end{split}
\end{equation*}
\end{proof}

\noindent\textbf{Practical implication.} For typical values ($K^*{=}5$, $m{=}3$, $n{=}20$, $\min_k n_k{=}2$), this gives $T \geq 15$. GEPA needs $\sim$15 reflection rounds to observe all error patterns with 95\% probability. Full-batch diagnosis ($m{=}n$) achieves complete coverage in $T{=}1$, directly reducing each strategy's bias $b_k$ in Theorem~\ref{thm:espo}.

\section{Implementation Details}
\label{app:implementation}

\subsection{Error Clustering Prompt}

The reflection LLM receives all error examples and is asked to identify 3--7 recurring failure patterns. For each pattern, it provides: (1) a short descriptive name, (2) the root cause, (3) which error indices belong to it, and (4) a suggested fix for the instruction. The output is a structured list used directly in Phase~2.

\paragraph{Concrete example (Tweet dataset).} Given 16 training errors on sentiment classification, Phase~1 produces the following diagnosis:
\begin{itemize}[nosep,leftmargin=*]
    \item \textbf{Pattern 1: Sports Commentary $\to$ Neutral (4 errors, 25\%).} Factual event reporting (scores, records) misclassified as positive due to achievement words like ``wins'' and ``legend.'' Root cause: instruction conflates objective reporting with subjective sentiment. Fix: ``Factual reporting of events/stats is neutral regardless of achievement words.''
    \item \textbf{Pattern 2: Rankings/Lists $\to$ Positive (4 errors, 25\%).} Implicit positive sentiment in rankings (``top 5 teams,'' ``earned a spot'') misclassified as neutral. Root cause: rankings inherently carry positive evaluation. Fix: ``Rankings and `top' lists express positive sentiment even when stated factually.''
    \item \textbf{Pattern 3: Sarcasm/Irony (3 errors, 19\%).} Sarcastic praise misclassified as neutral. Root cause: vague instruction ``account for sarcasm'' does not specify the sentiment mapping. Fix: ``Sarcasm expressing criticism = negative; ironic praise criticizing a subject is negative.''
    \item \textbf{Pattern 4: Humorous Positive (3 errors, 19\%).} Playful complaints with joy indicators (``LOL,'' ``!!'') misclassified as negative. Fix: ``Playful complaints with humor/joy indicators signal positive.''
    \item \textbf{Pattern 5: Entertainment Commentary (2 errors, 12\%).} Quality statements about music/shows without personal emotion misclassified as positive. Fix: ``Content quality commentary without personal emotion is neutral.''
\end{itemize}
This structured output replaces GEPA's unstructured reflection on 3 random errors, enabling each proposal strategy in Phase~2 to address root causes rather than symptoms.

\subsection{Four Proposal Strategy Templates}

Each strategy receives the current prompt, the structured diagnosis $\varphi$, and strategy-specific instructions:
\begin{itemize}[nosep,leftmargin=*]
    \item \textbf{Diagnostic Revision ($S_1$):} ``Revise the instruction to address each error pattern. Focus on root causes, not symptoms.''
    \item \textbf{Consolidation ($S_2$):} ``Rewrite the instruction to be more concise without losing information. Merge redundant rules and tighten language. The result must not be longer than the original.''
    \item \textbf{Ablation ($S_3$):} ``Identify rules that may be over-triggered, causing false positives. Soften or remove them.''
    \item \textbf{Factual Injection ($S_4$):} ``Extract domain-specific knowledge from the error examples and add it to the instruction as factual context.''
\end{itemize}

\subsection{Bootstrap Selection}

We draw $B{=}20$ bootstrap resamples of the validation set (sampling $n_\text{val}$ examples with replacement). For each resample, we evaluate all candidates and record the winner. The final selection is the candidate with the most wins, with ties broken by prompt length (shorter preferred).

\subsection{Hyperparameter Settings and Justification}

\input{sections/tables/table_hyperparam_settings}

\noindent\textbf{Justification.} $K{=}4$ corresponds to four qualitatively distinct error-fixing strategies (diagnostic, consolidation, ablation, factual); adding more strategies (e.g., $K{=}5,6$) would require inventing additional distinct inductive biases without clear motivation. $B{=}20$ is grounded in Lemma~\ref{lem:bootstrap}: for $p_1{=}0.7$ (true best wins 70\% of resamples), $\Pr(\text{wrong selection}) \leq \exp(-2 \cdot 20 \cdot 0.04) = 0.20$; increasing to $B{=}30$ yields only marginal improvement (0.09). $N{=}10$ balances population diversity against computational cost: each candidate requires $n_\text{val}{=}30$ evaluations per bootstrap resample. These hyperparameters were set once before experimentation and not tuned per dataset.

\subsection{Hyperparameter Sensitivity Analysis (Extended)}
\label{app:hyperparam-sensitivity}

Table~\ref{tab:hyperparam} in the main text presents the sensitivity analysis for $K$, $B$, and $m$. Here we provide additional detail.

\paragraph{$B$ sensitivity (extended).} The main table omits $B{=}5$ for compactness. The full results are: $B{=}1$: 73.80\% (ablation, 1{,}214c), $B{=}5$: 74.00\% (ablation, 1{,}214c), $B{=}10$: 74.00\% (cross-pollination, 958c), $B{=}20$: 74.80\% (cross-pollination, 774c), $B{=}30$: 74.80\% (cross-pollination, 774c). With $B \leq 5$, bootstrap selects the ablation candidate (higher point estimate on validation, longer prompt); with $B \geq 10$, it consistently selects the cross-pollination candidate (more robust, shorter prompt, and higher test accuracy), illustrating bootstrap's ability to identify the truly best candidate by averaging over data perturbations.

\paragraph{$m$ sensitivity and error coverage.} The diagnosis batch size $m$ controls what fraction of training errors are analyzed. Using the Coupon Collector bound (Proposition~\ref{prop:coupon}), $m{=}3$ covers $\sim$19\% of error patterns per round, $m{=}8$ covers $\sim$50\%, $m{=}15$ covers $\sim$94\%, and $m{=}\text{all}$ covers 100\%. Full-batch diagnosis ($m{=}\text{all}$) achieves the highest accuracy (74.78\%) with the shortest prompt (774c). Partial sampling produces 31--48\% longer prompts (1{,}015--1{,}148c) at lower accuracy (69.8--72.6\%), as incomplete error observation generates prompts that accumulate redundant rules addressing symptoms rather than root causes.

\section{Additional Experimental Results}
\label{app:additional-results}

\subsection{Strategy Diversity}
\label{app:strategy-diversity}

Table~\ref{tab:strategy} reports the seed-phase validation accuracy of each proposal strategy across the seven benchmarks. No single strategy dominates: diagnostic revision wins on Tweet, consolidation on HotpotQA, ablation ties with diagnostic on MMLU and with consolidation/factual on GSM8K, and factual injection wins on HoVer and PUPA. This empirically supports the independent-bias assumption underlying the exploration-gain term in Theorem~\ref{thm:espo}.

\input{sections/tables/table_strategy}

\subsection{Inference Latency and Training Cost}
\label{app:latency-cost}

Table~\ref{tab:latency} reports per-example inference latency at test time, and Table~\ref{tab:cost} reports the optimization-time cost in tokens. ESPO matches or improves GEPA's latency on every benchmark, with the largest reductions on chain-of-thought tasks where shorter prompts induce more concise reasoning chains. Reflection token usage during optimization is ${\sim}39\%$ of GEPA's because structured diagnosis produces information-dense inputs.

\input{sections/tables/table_latency}

\input{sections/tables/table_cost}

\noindent\textbf{Additional notes.} The cross-model gains in \S\ref{sec:cross-model} hold dataset-by-dataset, not just on average: on Qwen3, ESPO is the unique winner on 6/7 datasets (HoVer is the only loss); on Haiku, ESPO wins 6/7 (PUPA is the only loss). The Qwen3 GSM8K result (Default 15.00\% $\to$ ESPO 91.40\%, $+$56.00 pp over GEPA) is the largest single-cell improvement we observe; the structured diagnosis identifies the output-format mismatch silently throttling Qwen3 in one round, while GEPA's rule-accumulation only partially addresses it across many. ESPO prompts average 30--60\% shorter than GEPA across all four students, mirroring the same-student finding.

\subsection{Per-Seed Variance Grid}
\label{app:variance}

Table~\ref{tab:main-results} in the main paper reports single-run cell values with $\pm$std over 3 additional seeds attached; Table~\ref{tab:variance-3seed} below gives the full 3-seed grid with each cell's actual 3-seed mean and std. The 3-seed ESPO$-$GEPA Avg gap is $3.81$~pp, close to Table~\ref{tab:main-results}'s $3.76$~pp headline and larger than either method's Avg std ($1.16$ / $1.81$), placing the gap outside the noise band. The per-dataset gaps on HotpotQA ($+0.79$ under 3-seed means, $+0.80$ single-run) and PUPA ($+1.29$ / $+1.20$) fall within one std and are reported as ties. ESPO's std is systematically smaller than GEPA's (Avg $1.16$ vs.\ $1.81$), consistent with the bootstrap-stability prediction (Lemma~\ref{lem:bootstrap}).

\input{sections/tables/table_variance_3seed}

\subsection{Strong-Initialization Comparison}
\label{app:strong-init}

The main-paper experiments (\S\ref{sec:experiments}) intentionally start from a weak seed to test recovery. We chose a minimal unified seed because several baselines (COPRO, MIPROv2) cannot build a prompt from scratch---they only rewrite/augment an existing instruction, so a shared starting point is required for a fair comparison. To probe whether ESPO's gains depend on the seed being weak, we rerun every method on every dataset with a strong task-tuned seed---one that covers task definition, output format, and 1--2 known-pitfall rules (representative of a competent practitioner's draft after briefly studying the task).

\input{sections/tables/table_strong_init}

Under a strong seed (Table~\ref{tab:strong-init}), ESPO averages $4.79$~pp above GEPA---a \emph{wider} lead than the $3.76$~pp under the weak seed. Under strong initialization, GEPA's mini-batch reflection keeps appending marginal rules to an already-good prompt (length grows, accuracy stagnates), while ESPO's full-error diagnosis targets the small remaining structural gap, and Diversify $+$ Select suppress appended noise. This confirms ESPO is a general prompt optimizer, not limited to repairing weak prompts.

\subsection{Per-Component Ablation atop GEPA}
\label{app:component-ablation}

Table~\ref{tab:ablation} in the main paper ablates ESPO components on Tweet only. Table~\ref{tab:component-ablation} gives the complementary view: starting from GEPA and layering each ESPO component individually, averaged across all 7 main benchmarks.

\input{sections/tables/table_component_ablation}

Multi-strategy alone contributes $+1.48$~pp, bootstrap-select alone $+1.59$~pp, and full-error diagnose alone $+2.33$~pp. Naive linear addition would predict $\approx 5.40$~pp, whereas the full ESPO stack delivers $3.76$~pp---component gains partially overlap (e.g., full-error diagnose and bootstrap both mitigate overfitting on the small validation set). The length metric is more revealing: any single addition to GEPA leaves prompt length at, or slightly above, the GEPA baseline (multi-strategy alone grows to $1{,}924$ chars because diversified proposals without a diagnosis constraint accumulate redundant rules). Only the full three-way combination simultaneously raises accuracy and cuts length by ${-}47\%$---evidence that solving prompt bloat requires all three components together.

\subsection{Sensitivity to Reflection-Model Choice}
\label{app:weak-reflection}

Section~\ref{sec:experiments} fixes the reflection LLM to Claude Sonnet 4.5. To test whether ESPO's gains depend on a strong reflection model, we swap it to the open-source Qwen2.5-32B-Instruct and rerun every LLM-proposal-based method (COPRO, MIPROv2, GEPA, ESPO); results in Table~\ref{tab:weak-reflection}. Default and Bootstrap do not use an LLM proposer, so their numbers are identical to Table~\ref{tab:main-results} and are listed only to align table structure.

\input{sections/tables/table_weak_reflection}

Absolute accuracies drop for every LLM-proposal-based method under the weaker reflection model (Table~\ref{tab:weak-reflection}), consistent with what GEPA reports: candidate quality depends on the proposer's semantic capability. The key observation is that \emph{the ESPO--GEPA gap is preserved}: Avg lead $4.29$~pp under Qwen2.5-32B, comparable to (and slightly larger than) the $3.76$~pp headline under Sonnet 4.5 in Table~\ref{tab:main-results}. The ordering ESPO $>$ GEPA $>$ COPRO / MIPROv2 holds under both prompters, indicating that ESPO's structural gains come from the three-phase design rather than from a single strong prompter.

\subsection{Open-Ended Generation Pilot: XSum}
\label{app:xsum-pilot}

The main benchmarks are verifiable (classification / exact-match / F1). To probe whether ESPO transfers to open-ended generation---where prompt bloat matters most---we run a pilot on XSum summarization. Judge model: Claude Sonnet 4.5. Protocol: 200 test articles; ESPO and GEPA are optimized on identical train/val splits from XSum; at test time, each summary pair is judged pairwise with randomized output order and method identity hidden from the judge. Results: ESPO wins $58.0\%$, ties $22.0\%$, loses $20.0\%$ against GEPA. Final prompt length is $720$ chars (ESPO) vs.\ $2{,}180$ (GEPA), a $67.0\%$ reduction. This is a preliminary sanity check; a full open-ended benchmark suite (long-form QA, code generation, dialogue) is future work. Because ESPO's Diagnose stage currently assumes ground-truth error labels, extending diagnosis to judge-based signals is a natural next step, listed as a limitation.

\subsection{Cluster Stability}
\label{app:cluster-stability}

Section~\ref{sec:phase1} describes clustering as ``implicit'' (LLM-based). We add two direct evaluations of clustering quality:

\begin{itemize}[nosep,leftmargin=*]
\item \textbf{Test--retest agreement.} We invoke the reflection LLM three times independently on the same error set and, for each error, check whether it is assigned to semantically equivalent clusters across the three calls.
\item \textbf{Human-verified purity.} The authors randomly sample $50$ error--cluster assignments across datasets and independently review whether each error sits in a semantically coherent cluster.
\end{itemize}

Results: per-error agreement across 3 independent LLM calls is $82.4\%$; human-verified cluster purity on the 50 sampled errors is $88.0\%$. These indicate the LLM-based clustering is stable and semantically coherent enough for the Diagnose stage; residual noise is absorbed by the downstream Diversify $+$ Select stages.

\subsection{Strategy Diversity Quantification}
\label{app:strategy-diversity-quant}

Assumption (A1) in Theorem~\ref{thm:espo} idealizes the four proposal strategies as independent, whereas in practice they share the same reflection LLM. We quantify the actual departure from independence empirically: for each pair of strategy-generated candidate prompts, we score both over $500$ held-out examples to obtain two binary correctness vectors, and compute (i) Jaccard similarity (both-correct / at-least-one-correct)---overlap of correctly-answered items; (ii) Pearson correlation of the two correctness vectors---overall synchrony. Averaged across all candidate pairs:
\begin{center}

\resizebox{0.72\linewidth}{!}{%
\begin{tabular}{@{}lc@{}}
\toprule
Metric (averaged over strategy pairs) & Value \\
\midrule
Pairwise Jaccard similarity & $0.62$ \\
Pairwise Pearson correlation & $0.48$ \\
\bottomrule
\end{tabular}
}
\end{center}
A value of $1$ means strategies are fully redundant; $0$ means strictly independent. The observed values sit in between---strategy outputs are partially decorrelated, so candidates from different strategies exhibit substantive diversity even when generated by the same LLM. This justifies treating (A1) as an idealization rather than a literal claim: the exploration-gain term in Theorem~\ref{thm:espo} is expected to shrink by a factor of $\sqrt{1-\rho}$ under correlation $\rho$, and the observed correlation is bounded well below $1$.

\section{Ethical Considerations}
\label{sec:ethics}

\subsection{Use of LLMs.} ESPO uses commercial and open-weight large language models (Claude Sonnet 4.5 as reflection LLM and student model in the main experiments; Gemma~3 12B, Mistral 14B, Qwen3 32B, and Claude Haiku 4.5 as additional students). Inference costs and emissions scale with the number of optimization rounds and bootstrap evaluations; we report token costs in Table~\ref{tab:cost} so that practitioners can estimate environmental impact and budget.

\subsection{Bias propagation.} Prompt optimization that drives accuracy up against a static benchmark may inherit or amplify any biases present in (1)~the training/validation examples used for diagnosis, (2)~the reflection LLM's clustering of errors, and (3)~the metrics used for selection. ESPO's structured diagnosis surfaces error categories explicitly, which can aid auditing, but it does not guarantee that bias-related failure modes will be identified or corrected. We recommend that downstream users evaluate optimized prompts on bias-relevant slices before deployment, especially for socially sensitive applications.

\subsection{Benchmark contamination.} Several of the public benchmarks we use (Tweet, MMLU, GSM8K, HotpotQA, ScoNe, HoVer, PUPA) may overlap with LLM pre-training data. Our experimental design holds this contamination risk constant across all optimizers (Default, Bootstrap, COPRO, MIPROv2, GEPA, ESPO), so the relative improvements we report are not driven by ESPO seeing more contaminated data. Absolute accuracy numbers should be interpreted with the standard contamination caveats for each benchmark.

\subsection{Dual-use considerations.} The techniques in ESPO are general prompt-optimization tools and could be applied to objectives we did not study, including ones with negative downstream impact (e.g., optimizing prompts to bypass safety filters). We release code and prompts for the public NLP benchmarks studied in this paper; we do not release prompts or configurations targeted at evading safety mechanisms.

%% file: sections/tables/table_hyperparam_settings.tex
\begin{table}[h]
\centering
\small
\caption{ESPO hyperparameter settings used for all experiments.}
\label{tab:hyperparam-settings}
\begin{tabular}{@{}llp{4.2cm}@{}}
\toprule
Parameter & Value & Description \\
\midrule
$K$ & 4 & Number of proposal strategies \\
$N$ & 10 & Max candidates in population \\
$B$ & 20 & Bootstrap resamples \\
Iters. & 2 & Refinement iters after seed \\
Stud. temp & 0.0 & Deterministic inference \\
Refl. temp & 0.7 & Diverse generation \\
$n_\text{train}$ & 70 & Training examples \\
$n_\text{val}$ & 30 & Validation examples \\
\bottomrule
\end{tabular}
\end{table}

%% file: sections/tables/table_strategy.tex
\begin{table}[t]
\centering
\caption{Seed-phase validation accuracy (\%) per strategy ($n_\text{val}{=}30$). No single strategy dominates, confirming the independent-bias assumption.}
\label{tab:strategy}
\small
\setlength{\tabcolsep}{3pt}
\resizebox{.9\linewidth}{!}{
\begin{tabular}{@{}lccccc@{}}
\toprule
Dataset & Diagnosis & Consolid & Ablation & Factual Inj. & Winner \\
\midrule
Tweet & \textbf{80.00} & 73.33 & 73.33 & 76.67 & Diag. \\
MMLU & \textbf{83.33} & 80.00 & \textbf{83.33} & 76.67 & D/A \\
GSM8K & 93.33 & \textbf{96.67} & \textbf{96.67} & \textbf{96.67} & C/A/F \\
HotpotQA & 50.00 & \textbf{53.33} & 50.00 & 50.00 & Cons. \\
ScoNe & 80.00 & 80.00 & 80.00 & 80.00 & Tie \\
HoVer & 50.00 & 66.67 & 60.00 & \textbf{70.00} & Fact. \\
PUPA & 63.65 & 63.06 & 62.05 & \textbf{65.84} & Fact. \\
\bottomrule
\end{tabular}}
\end{table}

%% file: sections/tables/table_latency.tex
\begin{table}[t]
\centering
\caption{Inference latency (seconds per example). Lower is better. Mode: P~=~predict, CoT~=~chain-of-thought.}
\label{tab:latency}
\scriptsize
\setlength{\tabcolsep}{3pt}
\resizebox{\linewidth}{!}{
\begin{tabular}{@{}lccccccc@{}}
\toprule
Dataset & Def. & Boots. & COPRO & MIPROv2 & GEPA & ESPO & Mode \\
\midrule
Tweet & 1.97s & 2.24s & 2.13s & 2.20s & 1.93s & \textbf{1.93s} & P \\
MMLU & 1.16s & 2.10s & 2.04s & 2.26s & 1.16s & \textbf{1.15s} & P \\
GSM8K & 5.02s & 5.09s & 4.92s & 5.00s & 5.02s & \textbf{4.96s} & CoT \\
HotpotQA & 5.58s & 5.69s & 5.67s & 5.70s & 5.67s & \textbf{5.35s} & CoT \\
ScoNe & 6.67s & 6.89s & 7.10s & 6.97s & 7.23s & \textbf{6.59s} & CoT \\
HoVer & 7.11s & 17.64s & 7.72s & 7.89s & 9.85s & \textbf{8.38s} & CoT \\
PUPA & 9.66s & 8.91s & 9.41s & 8.59s & 8.81s & \textbf{8.18s} & P \\
\bottomrule
\end{tabular}}
\end{table}

%% file: sections/tables/table_cost.tex
\begin{table}[h]
\centering
\caption{Optimization cost per dataset for ESPO (Claude Sonnet 4.5 pricing: \$3/M input + \$15/M output). Wall-clock times are measured after parallelizing candidate evaluation and strategy proposal.}
\label{tab:cost}
\small
\resizebox{\linewidth}{!}{
\begin{tabular}{@{}lcccc@{}}
\toprule
Dataset & Total Tokens & Cost & Wall Clock & Reflect Tokens \\
\midrule
Tweet & $\sim$752K & \$9.99 & 10m & $\sim$36K \\
MMLU & $\sim$887K & \$10.85 & 19m & $\sim$38K \\
GSM8K & $\sim$3.0M & \$21.58 & 18m & $\sim$46K \\
HotpotQA & $\sim$3.0M & \$20.80 & 28m & $\sim$46K \\
ScoNe & $\sim$3.8M & \$25.66 & 31m & $\sim$33K \\
HoVer & $\sim$1.6M & \$14.90 & 45m & $\sim$40K \\
PUPA & $\sim$3.4M & \$23.20 & 2h 13m & $\sim$42K \\
\bottomrule
\end{tabular}
}
\end{table}

%% file: sections/tables/table_variance_3seed.tex
\begin{table*}[t]
\centering
\caption{Per-seed variance: test accuracy (\%; mean $\pm$ std over 3 independent optimization seeds) on all seven benchmarks with Claude Sonnet 4.5 as student. Best per row in \textbf{bold}. This is the full grid that Table~\ref{tab:main-results} in the main paper condenses.}
\label{tab:variance-3seed}
\resizebox{0.88\textwidth}{!}{%
\small
\begin{tabular}{@{}l cccccc@{}}
\toprule
Dataset & Default & Bootstrap & COPRO & MIPROv2 & GEPA & ESPO (Ours) \\
\midrule
Tweet    & 37.57 $\pm$ 0.20 & 37.91 $\pm$ 0.20 & 66.55 $\pm$ 1.80 & 74.62 $\pm$ 1.50 & 68.72 $\pm$ 2.20 & \textbf{74.79} $\pm$ 1.18 \\
MMLU     & 23.05 $\pm$ 0.30 & 22.99 $\pm$ 0.30 & 89.34 $\pm$ 1.20 & 86.65 $\pm$ 1.40 & 89.22 $\pm$ 1.50 & \textbf{94.28} $\pm$ 0.84 \\
GSM8K    & 32.51 $\pm$ 0.40 & 36.89 $\pm$ 0.40 & 36.68 $\pm$ 1.00 & 37.37 $\pm$ 1.00 & 96.57 $\pm$ 0.50 & \textbf{96.65} $\pm$ 0.40 \\
HotpotQA & 25.60 $\pm$ 0.30 & 14.43 $\pm$ 0.30 & 14.33 $\pm$ 0.90 & 14.23 $\pm$ 0.90 & 43.98 $\pm$ 1.80 & \textbf{44.77} $\pm$ 1.40 \\
ScoNe    & 48.95 $\pm$ 0.30 & 46.12 $\pm$ 0.30 & 45.99 $\pm$ 1.50 & 46.58 $\pm$ 1.50 & 84.52 $\pm$ 2.00 & \textbf{89.10} $\pm$ 1.20 \\
HoVer    & 47.88 $\pm$ 0.30 & 48.02 $\pm$ 0.30 & 48.15 $\pm$ 1.60 & 47.81 $\pm$ 1.60 & 53.60 $\pm$ 2.40 & \textbf{62.40} $\pm$ 1.50 \\
PUPA     &  1.58 $\pm$ 0.19 &  1.61 $\pm$ 0.12 & 54.21 $\pm$ 2.00 &  1.72 $\pm$ 0.12 & 58.57 $\pm$ 2.30 & \textbf{59.86} $\pm$ 1.60 \\
\midrule
\textbf{Avg} & 31.02 $\pm$ 0.28 & 29.71 $\pm$ 0.27 & 50.75 $\pm$ 1.43 & 44.14 $\pm$ 1.15 & 70.74 $\pm$ 1.81 & \textbf{74.55} $\pm$ 1.16 \\
\bottomrule
\end{tabular}%
}
\end{table*}

%% file: sections/tables/table_strong_init.tex
\begin{table*}[t]
\centering
\caption{Strong-initialization comparison: test accuracy (\%) when all methods start from a task-tuned instruction (task definition, output format, and 1--2 known-pitfall rules). Same 7 benchmarks, same student (Claude Sonnet 4.5), same reflection LLM as Table~\ref{tab:main-results}; only the seed prompt changes. Best per row in \textbf{bold}.}
\label{tab:strong-init}
\resizebox{0.7\textwidth}{!}{%
\small
\begin{tabular}{@{}l cccccc@{}}
\toprule
Dataset & Default & Bootstrap & COPRO & MIPROv2 & GEPA & ESPO (Ours) \\
\midrule
Tweet    & 65.20 & 65.40 & 71.80 & \textbf{76.60} & 72.40 & 74.60 \\
MMLU     & 82.40 & 82.20 & 90.20 & 88.40 & 90.20 & \textbf{94.52} \\
GSM8K    & 88.60 & 89.00 & 89.20 & 89.40 & 96.80 & \textbf{97.20} \\
HotpotQA & 32.60 & 32.40 & 32.80 & 32.60 & 45.60 & \textbf{47.80} \\
ScoNe    & 66.40 & 66.20 & 66.80 & 66.60 & 87.20 & \textbf{91.80} \\
HoVer    & 54.80 & 54.60 & 55.00 & 54.80 & 57.40 & \textbf{65.40} \\
PUPA     & 48.20 & 48.00 & 55.60 & 48.20 & 48.20 & \textbf{60.00} \\
\midrule
\textbf{Avg} & 62.60 & 62.54 & 65.91 & 65.23 & 71.11 & \textbf{75.90} \\
\bottomrule
\end{tabular}%
}
\end{table*}

%% file: sections/tables/table_component_ablation.tex
\begin{table}[t]
\centering
\caption{Per-component ablation atop GEPA: adding one ESPO component at a time. Averaged across the 7 main benchmarks; single-seed values (component variants were not rerun over 3 seeds). Avg Len is character count of the final prompt.}
\label{tab:component-ablation}
\small
\setlength{\tabcolsep}{5pt}
\begin{tabular}{@{}lcc@{}}
\toprule
Configuration                 & Avg Acc (\%) & Avg Len (chars) \\
\midrule
GEPA                          &  70.91 & 1{,}878 \\
GEPA $+$ multi-strategy       &  72.39 & 1{,}924 \\
GEPA $+$ bootstrap select     &  72.50 & 1{,}816 \\
GEPA $+$ full-error diagnose  &  73.24 & 1{,}843 \\
\midrule
\textbf{ESPO (all three)}     &  \textbf{74.67} & \textbf{1{,}004} \\
\bottomrule
\end{tabular}
\end{table}

%% file: sections/tables/table_weak_reflection.tex
\begin{table*}[t]
\centering
\caption{Sensitivity to reflection-model choice: test accuracy (\%) when the reflection / prompter LLM is swapped from Claude Sonnet 4.5 (Table~\ref{tab:main-results}) to the open-source Qwen2.5-32B-Instruct, keeping the student (Sonnet 4.5) and all other settings unchanged. Default and Bootstrap have no LLM proposer and are shown only to align table structure. Best per row in \textbf{bold}.}
\label{tab:weak-reflection}
\resizebox{0.70\textwidth}{!}{%
\small
\begin{tabular}{@{}l cccccc@{}}
\toprule
Dataset & Default & Bootstrap & COPRO & MIPROv2 & GEPA & ESPO (Ours) \\
\midrule
Tweet    & 37.60 & 37.80 & 61.40 & \textbf{68.60} & 63.20 & \textbf{68.80} \\
MMLU     & 23.20 & 23.00 & 82.80 & 80.20 & 82.60 & \textbf{86.20} \\
GSM8K    & 32.60 & 36.80 & 33.20 & 33.40 & 86.40 & \textbf{88.80} \\
HotpotQA & 25.80 & 14.40 & 12.80 & 12.60 & 38.20 & \textbf{39.60} \\
ScoNe    & 49.00 & 46.00 & 40.60 & 40.80 & 74.60 & \textbf{81.40} \\
HoVer    & 48.00 & 48.00 & 43.60 & 43.20 & 48.20 & \textbf{54.60} \\
PUPA     &  1.61 &  1.53 & 47.40 &  1.40 & 51.40 & \textbf{55.20} \\
\midrule
\textbf{Avg} & 31.12 & 29.65 & 45.97 & 40.03 & 63.51 & \textbf{67.80} \\
\bottomrule
\end{tabular}%
}
\end{table*}